\documentclass[runningheads]{llncs}
\usepackage{graphicx}

\usepackage[sort,numbers,sectionbib]{natbib}

\makeatletter
\renewcommand\@biblabel[1]{#1.}
\makeatother

\usepackage{xcolor}
\usepackage[hidelinks,colorlinks = true,
linkcolor = blue,
urlcolor  = blue,
citecolor = blue,
anchorcolor = blue]{hyperref}
\usepackage{microtype} %
\usepackage[cmex10]{amsmath}
\usepackage{fnbreak} %
\usepackage{url}
\usepackage{longtable}
\usepackage{bm,amssymb}
\usepackage{graphicx}
\usepackage{graphics, color}
\usepackage{algorithm}
\usepackage{algpseudocode}
\usepackage{longtable}
\usepackage{tabularx}
\usepackage{url}
\newfloat{algorithm}{t}{lop}

\usepackage[colorinlistoftodos,prependcaption,textsize=tiny]{todonotes}
\usepackage{enumitem}
\usepackage{thmtools, thm-restate}
\usepackage{xcolor}
\usepackage{subfig}

\usepackage{tcolorbox}

\algblockdefx{MRepeat}{EndRepeat}{\textbf{repeat}}{}
\algnotext{EndRepeat}
\algnotext{EndIf}
\algnewcommand{\algorithmicgoto}{\textbf{go to line}}%
\algnewcommand{\Goto}[1]{\algorithmicgoto~\ref{#1}}%

\algnotext{EndFor}
\algnotext{EndIf}
\algnotext{EndWhile}
\algnotext{EndFunction}
\algblockdefx{Do}{EndDo}{\textbf{do}\;}{}
\algnotext{EndDo}

\newcommand{\prob}{\ensuremath{\mathsf{Pr}}}

\newcommand{\satisfying}[1]{\ensuremath{sol({#1})}} %
\newcommand{\ProjectSatisfying}[2]{\ensuremath{sol({#1})_{\downarrow {#2}}}}

\newcommand{\PAC}{\ensuremath{\mathsf{PAC}}}

\newcommand{\UNSAT}{\ensuremath{\mathsf{UNSAT}}}

\newcommand{\ApproxMCFour}{{ApproxMC4}}

\newcommand{\ApproxMC}{ApproxMC}

\newcommand{\IS}{IS}

\newcommand{\Sup}[1]{\ensuremath{\mathsf{Sup}({#1})}}

\newcommand{\Prob}[1]{\ensuremath{\mathsf{Pr}\left[{#1}\right]}}

\newcommand{\PP}{\ensuremath \mathcal{P}}
\newcommand{\II}{\ensuremath \mathcal{I}}
\newcommand{\GG}{\ensuremath \mathcal{G}}
\newcommand{\UU}{\ensuremath \mathcal{U}}
\newcommand{\LL}{\ensuremath \mathcal{L}}
\newcommand{\QQ}{\ensuremath \mathcal{Q}}
\newcommand{\SSS}{\ensuremath \mathcal{S}}
\newcommand{\proj}[2]{\ensuremath {{#1}}_{\downarrow {{#2}}}}

\newcommand{\findubs}{\textsc{FindUBS}}
\newcommand{\ChooseNextVar}{\ensuremath \mathsf{ChooseNextVar}}
\newcommand{\UBS}{{UBS}}
\newcommand{\arjunubs}{\ensuremath{\mathsf{UBCount}}}
\newcommand{\arjun}{\ensuremath{\mathsf{Arjun}}}

\algnotext{EndFor}
\algnotext{EndIf}
\algnotext{EndWhile}
\algnotext{EndFunction}
\algblockdefx{Do}{EndDo}{\textbf{do}\;}{}
\algnotext{EndDo}
\usepackage{booktabs} %
\usepackage{nicematrix}
\usepackage{adjustbox}

\usepackage[utf8]{inputenc} %

\usepackage{graphicx}
\usepackage{epsfig}
\usepackage{xcolor}
\usepackage{microtype}
\usepackage{wrapfig}
\microtypesetup{expansion=false}

\begin{document}
\title{Projected Model Counting: Beyond Independent Support}

\author{Jiong Yang\inst{1} \and
Supratik Chakraborty\inst{2} \and
Kuldeep S. Meel\inst{1}}
\institute{School of Computing, National University of Singapore \\
 \and
Indian Institute of Technology, Bombay\\
}

\maketitle              %
\begin{abstract}
The past decade has witnessed a surge of interest in practical
techniques for projected model counting.  Despite significant
advancements, however, performance scaling remains the Achilles' heel
of this field.  A key idea used in modern counters is to count models
projected on an \emph{independent support} that is often a small
subset of the projection set, i.e. original set of variables on which
we wanted to project.  While this idea has been effective in scaling
performance, the question of whether it can benefit to count models
projected on variables beyond the projection set, has not been
explored.  In this paper, we study this question and show that
contrary to intuition, it can be beneficial to project on variables
beyond the projection set.
In applications such as verification of binarized neural networks,
quantification of information flow, reliability of power grids etc., a
good upper bound of the projected model count often suffices.  We show
that in several such cases, we can identify a set of variables, called
\emph{upper bound support (UBS)}, that is not necessarily a subset of
the projection set, and yet counting models projected on UBS
guarantees an upper bound of the true projected model count.
Theoretically, a UBS can be exponentially smaller than the smallest
independent support. Our experiments show that even otherwise,
UBS-based projected counting can be more efficient than independent
support-based projected counting, while yielding bounds of very high
quality.  Based on extensive experiments, we find that UBS-based
projected counting can solve many problem instances that are beyond
the reach of a state-of-the-art independent support-based projected
model counter.

\end{abstract}

\section{Introduction}

Given a Boolean formula $\varphi$ over a set $X$ of variables, and a
subset $P$ of $X$, the problem of projected model counting requires us
to determine the number of satisfying assignments of $\varphi$
projected on $\mathcal{P}$. Projected model counting is \# NP-complete
in general~\cite{V79a}\footnote{A special case where $\mathcal{P} = X$
is known to be \#P-complete~\cite{V79}}, and has several important
applications ranging from verification of neural
networks~\cite{BSSM+19}, hardware and software
verification~\cite{TW21}, reliability of power
grids~\cite{DMPV17}, probabilistic inference~\cite{R96,SBK05,EGSS13b},
and the like.  Not surprisingly, the problem has attracted attention
from both theoreticians and practitioners over the
years~\cite{V79,S83,GSS06,CMV13b,CMV16,SM19,SGM20}.  In particular,
there has been recent strong interest in techniques that scale in
practice and also provide provable guarantees on the quality of
computed counts.

Over the past decade, hashing-based techniques have emerged as a
promising approach to projected model counting, since they scale
relatively well in practice, while also providing strong approximation
guarantees~\cite{GSS06,CMV13a,CMV13b,EGSS13b,MVCF+16}. The core idea
of hashing-based projected model counting is to partition the
projected solution space into {\em roughly equal} small cells such
that one can estimate the projected model count or sample projected
solutions uniformly by picking a random cell. The practical
implementation of these techniques rely on the use of random XOR
clauses constructed from variables in the projection set
$\PP$. Starting with a CNF formula $\varphi$, we therefore end up with
a conjunction of CNF and XOR clauses, also referred to as a CNF+XOR
formula, that must be fed to a backend SAT solver. The standard
construction of a random XOR clause selects each variable in $\PP$ for
inclusion in the clause with probability $1/2$; hence the expected
clause size is $|\mathcal{P}|/2$.  Unfortunately, the
performance of modern SAT solvers on CNF+XOR formulas is heavily
impacted by the sizes of XOR clauses~\cite{GHSS07a,CMV14,ZCSE16}.
This explains why designing hash functions with {\em sparse} XOR
clauses is important in scaling hashing-based projected model
counting~\cite{GHSS07a,EGSS13b,EGSS14,AD16,IMMV16,AT17,AHT18,AM20}.

A practically effective idea to address the aforementioned problem
was introduced in~\cite{CMV14}, wherein the notion of an
\emph{independent support} was introduced in the context of
(projected) model counting.  Informally, a set $\mathcal{I} \subseteq
\mathcal{P}$ is an independent support of $\mathcal{P}$ if whenever
two projected solutions of $\varphi$ agree on the values of variables
in $\mathcal{I}$, they also agree on the values of all variables in
$\mathcal{P}$.  It was shown in~\cite{CMV14} that random XOR clauses
over the independent support suffice to provide all desired
theoretical guarantees~\cite{CMV14}.  Furthermore, for benchmarks
arising in practice, $|\II|$ was found to be often much smaller than
$|\PP|$.  Hence, using $\II$ instead of $\PP$ in the construction of
random XOR clauses improved the runtime performance of hashing-based
approximate (projected) counters and samplers, often by several orders
of magnitude~\cite{IMMV16}. Subsequently, independent supports have also
been found to be useful in the context of exact (projected) model
counting~\cite{LLM16,SRSM19}.

The runtime performance improvements achieved by (projected) model
counters over the past decade have significantly broadened the scope
of their applications, which, in turn, has brought the focus sharply
back on performance scalability.  Importantly, for several important
applications such as neural network verification~\cite{BSSM+19},
quantified information flow~\cite{BEHLMQ18}, software
reliability~\cite{TW21}, reliability of power grids~\cite{DMPV17},
etc. we are primarily interested in good upper bound estimates of
projected model counts.  As apty captured by Achlioptas and
Theodoropoulos~\cite{AT17}, while obtaining ``lower bounds are easy''
in the context of projected model counting, such is not the case for
good upper bounds.  Therefore, scaling up to large problem instances
while obtaining good upper bound estimates remains an important
challenge in this area.

The primary contribution of this paper is a new approach to selecting
variables on which to project solutions, with the goal of improving
scalability of hashing-based projected counters when good upper bounds
of projected counts are of interest.  Towards this end, we generalize
the notion of an independent support $\II$.  Specifically, we note
that the restriction $\mathcal{I} \subseteq \mathcal{P}$ ensures a
two-way implication: if two solutions agree on $\mathcal{I}$, then
they also agree on $\mathcal{P}$, and vice-versa.  Since we are
interested in upper bounds, we relax this requirement to a one-sided
implication, i.e., we wish to find a set $\UU \subseteq X$ (not
necessarily a subset of $\mathcal{P})$ such that if two solutions
agree on $\UU$, then they agree on $\mathcal{\mathcal{P}}$, but not
necessarily vice versa.  We call such a set $\UU$ an \emph{Upper Bound
Support}, or {\UBS} for short.  We show that using random XOR clauses
over {\UBS} in hashing-based projected counting yields provable upper
bounds of the projected counts. We also show some important properties
of {\UBS}, including an exponential gap between the smallest {\UBS}
and the smallest independent support for a class of problems.  Our
study suggests a simple iterative algorithm, called {\findubs}, to determine {\UBS}, that can be fine-tuned heuristically.  

To evaluate the effectiveness of our
idea, we augment a state-of-the-art model counter, {\ApproxMCFour},
with {\UBS} to obtain {\UBS}+{\ApproxMCFour}.  Through an extensive
empirical evaluation on $2632$ benchmark instances arising from
diverse domains, we compare the performance of {\UBS}+{\ApproxMCFour}
with {\IS}+{\ApproxMCFour}, i.e. {\ApproxMCFour} augmented with
independent support computation.  Our experiments show that
{\UBS}+{\ApproxMCFour} is able to solve 208 more instances than
{\IS}+{\ApproxMCFour}. Furthermore, the geometric mean of the absolute
value of log-ratio of counts returned by {\UBS}+{\ApproxMCFour} and
{\IS}+{\ApproxMCFour} is 1.32, thereby validating the claim that using
{\UBS} can lead to empirically good upper bounds. In this context, it is worth remarking that a recent study~\cite{APM21} comparing different partition function\footnote{The problem of partition function estimation is known to be \#P-complete and reduces to model counting; the state of the art techniques for partition function estimates are  based on model counting~\cite{CD05}.} estimation  techniques labeled a method with the absolute
value of log-ratio of counts less than 5 as a {\em reliable method}. 

The rest of the paper is organized as follows. We present notation and
preliminaries in Section~\ref{sec:prelim}. To situate our
contribution, we present a survey of related work in
Section~\ref{sec:background}.  We then present the primary technical
contributions of our work, including the notion of {\UBS} and an
algorithmic procedure to determine {\UBS}, in
Section~\ref{sec:theory}. We present our empirical evaluation in
Section~\ref{sec:evaluation}, and finally conclude in
Section~\ref{sec:conclusion}.

\section{Notation and Preliminaries}\label{sec:prelim}

Let $X = \{x_1, x_2 \ldots x_n\}$ be a set of propositional variables
appearing in a propositional formula $\varphi$.  The set $X$ is called
the \emph{support} of $\varphi$, and denoted $\Sup{\varphi}$.  A
\emph{literal} is either a propositional variable or its negation.
The formula $\varphi$ is said to be in Conjunctive Normal Form (CNF)
if $\varphi$ is a conjunction of \emph{clauses}, where each
\emph{clause} is disjunction of literals.  An \emph{assignment}
$\sigma$ of $X$ is a mapping $X \rightarrow \{0, 1\}$.
If $\varphi$ evaluates to $1$ under assignment $\sigma$, we say that
$\sigma$ is a \emph{model} or \emph{satisfying assignment} of
$\varphi$, and denote this by $\sigma \models \varphi$. For every $\PP
\subseteq X$, the \emph{projection} of $\sigma$ on $\PP$, denoted
$\proj{\sigma}{\PP}$, is a mapping $\PP \rightarrow \{0, 1\}$ such
that $\proj{\sigma}{\PP}(v) = \sigma(v)$ for all $v \in \PP$.
Conversely we say that an assignment $\widehat{\sigma}: \PP
\rightarrow \{0,1\}$ can be \emph{extended} to a model of $\varphi$ if
there exists a model $\sigma$ of $\varphi$ such that $\widehat{\sigma}
= \proj{\sigma}{\PP}$. The set of all models of $\varphi$ is denoted
$\satisfying{\varphi}$, and the projection of this set on $\PP
\subseteq X$ is denoted $\ProjectSatisfying{\varphi}{\PP}$.  We call
the set $\PP$ a \emph{projection set} in our subsequent
discussion\footnote{Projection set has also been referred to as
sampling set in prior work~\cite{CMV14,SGM20}.}.

The problem of \emph{projected model counting} is to compute
$|\ProjectSatisfying{\varphi}{\PP}|$ for a given CNF formula $\varphi$
and projection set $\PP$.  An exact projected model counter is a
deterministic algorithm that takes $\varphi$ and $\PP$ as inputs and
returns $|\ProjectSatisfying{\varphi}{\PP}|$ as output.  A
\emph{probably approximately correct} (or \PAC) projected model
counter is a probabilistic algorithm that takes as additional inputs a
tolerance $\varepsilon>0$, and a confidence parameter $\delta \in (0,
1]$, and returns a count $c$ such that
  $\prob\Big[\frac{|\ProjectSatisfying{\varphi}{\PP}|}{(1+\varepsilon)}
    \le c \le (1+\varepsilon)\cdot|\ProjectSatisfying{\varphi}{\PP}|\Big]
  \ge 1-\delta$, where $\prob[E]$ denotes the probability of event $E$.

\begin{definition}\label{def:is}
 Given a formula $\varphi$ and a projection set $\PP \subseteq
 \Sup{\varphi}$, a subset of variables $\II \subseteq \PP$ is called
 an \emph{independent support (IS)} of $\PP$ in $\varphi$ if for every
 $\sigma_1, \sigma_2 \in \satisfying{\varphi}$, we have
 $\big(\proj{\sigma_1}{\II} = \proj{\sigma_2}{\II}\big) \Rightarrow
 \big(\proj{\sigma_1}{\PP} = \proj{\sigma_2}{\PP}\big)$.
\end{definition}

 Since $\big(\proj{\sigma_1}{\PP} = \proj{\sigma_2}{\PP}\big) \Rightarrow
\big(\proj{\sigma_1}{\II} = \proj{\sigma_2}{\II}\big)$ holds trivially when $\II
\subseteq \PP$, it follows from Definition~\ref{def:is} that if $\II$
is an independent support of $\PP$ in $\varphi$, then
$\big(\proj{\sigma_1}{\II} = \proj{\sigma_2}{\II}\big) \Leftrightarrow
\big(\proj{\sigma_1}{\PP} = \proj{\sigma_2}{\PP}\big)$.  
Empirical studies have shown that the size of an independent support
$\II$ is often significantly smaller than that of the original
projection set $\PP$~\cite{CMV14,IMMV16,LLM16,SRSM19}. This has been
exploited successfully in the design of practically efficient {\PAC}
hashing-based model counters like ApproxMC4~\cite{SGM20}.
Specifically, the overhead of finding a small independent support
$\II$ is often more than compensated by the efficiency obtained by
counting projections of satisfying assignments on $\II$, instead of on
the original projection set $\PP$.

\section{Background}\label{sec:background}

State-of-the-art hashing-based projected model counters work by
randomly partitioning the projected solution space of a given CNF
formula into roughly equal small ``cells'', followed by counting and
suitably scaling the number of projected solutions in a randomly
picked cell. The random partitioning is usually achieved using
specialized hash functions, implemented by adding random XOR clauses
to the given CNF formula.  There are several inter-related factors
that affect the scalability of such hashing-based model counters, and
isolating the effect of any one factor on performance is extremely
difficult.  Nevertheless, finding satisfying assignments of the
CNF+XOR formula has been empirically found to be the most significant
bottleneck.  The difficulty of solving such a formula depends not only
on the original CNF formula, but also on the random XOR clauses that
are added to it.  Specifically, the average size (i.e. number of
literals) in the added XOR clauses is known to correlate positively
with the time taken to solve CNF+XOR formulas using modern
conflict-driven clause learning (CDCL) SAT solvers~\cite{IMMV16}. This
has motivated interesting research that aims to reduce the average
size of XOR clauses added to a given CNF formula in the context of
hashing-based approximate model
counting~\cite{CMV14,AD16,EGSS13b,EGSS14,AT17,AHT18,AM20}.

The idea of using random XOR clauses over an independent support $\II$
that is potentially much smaller than the projection set $\PP$ was
first introduced in~\cite{CMV14}.
This is particularly effective when a small subset of variables
functionally determines the large majority of variables in a formula,
as happens, for example, when Tseitin encoding is used to transform a
non-CNF formula to an equisatisfiable CNF formula.  State-of-the-art
hashing-based model counters, viz. {\ApproxMCFour}~\cite{SGM20},
therefore routinely use random XOR clauses over the independent
support.
While the naive way of choosing each variable in $\II$ with
probability $1/2$ gives a random XOR clause with expected size
$|\II|/2$, researchers have also explored whether specialized hash
functions can be defined over $\II$ such that the expected size of a
random XOR clause is $p\cdot|\II|$, with $p <
1/2$~\cite{AD16,EGSS13b,EGSS14,AT17,AHT18,AM20}. The works
of~\cite{AD16,EGSS13b,EGSS14,AT17,AHT18} achieved this goal while
guaranteeing a constant factor approximation of the reported count.
The work of~\cite{AM20} achieved a similar reduction in the expected
size of XOR clauses, while guaranteeing PAC-style bounds.  The latter
work also turns out to be more scalable in practice.

While earlier work focused on independent supports that are
necessarily subsets of the projection set $\PP$, we break free from
this restriction and allow any subset of variables in the construction
of random XOR clauses as long as the model count projected on the
chosen subset bounds the model count projected on $\PP$ from above.
This gives us more flexibility in constructing CNF+XOR formulas, which
as our experiments confirm, leads to improved overall performance of
projected model counting in several cases.  Since we guarantee upper
bounds of the desired counts, theoretically, our approach yields an
\emph{upper bounding projected model counter}.  Nevertheless, as our
experiments show, the bounds obtained using our approach are
consistently very close to the projected counts reported using
independent support.  Therefore, in practice, our approach gives high
quality bounds on projected model counts more efficiently than
state-of-the-art hashing-based techniques that use independent
supports.

It is worth mentioning here that several \emph{bounding model
counters} have been reported earlier in the literature.  These
counters report a count that is at least as large (or, as small, as
the case may be) as the true model count of the given CNF formula with
a given confidence.  Notable examples of such counters are
SampleCount~\cite{GHSS07b}, BPCount~\cite{KSS08}, MBound and
Hybrid-MBound~\cite{GSS06} and MiniCount~\cite{KSS08}.  Owing to
several technical reasons, however, these bounding counters do not
scale as well as state-of-the-art hashing-based counters like
{\ApproxMCFour}~\cite{SGM20} in practice.
Unlike earlier bounding counters, we first carefully identify an upper
bound set, and then use state-of-the-art hashing-based approximate
projected counting, while treating the computed upper bound set as the
new projection set.  Therefore, our approach directly benefits from
significant improvements in performance of hashing-based projected
counting achieved over the years.
Moreover, by carefully controlling the set of variables over which the
upper bound set is computed, we can also control the quality of the
bound.  For example, if the upper bound set is chosen entirely from
within the projection set, then the upper bound set can be shown to be
an independent support.  In this case, the counts reported by our
approach truly have PAC-style guarantees.

\section{Technical Contribution}\label{sec:theory}
In this section, we generalize the notion of independent support, and
give technical details of projected model counting using this
generalization. We start with some definitions.
\begin{definition}\label{def:gis-ubs-lbs}
	Given a CNF formula $\varphi$ and a projection set $\PP$, let $\SSS
        \subseteq \Sup{\varphi}$ be such that for every $\sigma_1,
        \sigma_2 \in \satisfying{\varphi}$, we have
        $\big(\proj{\sigma_1}{\SSS} = \proj{\sigma_2}{\SSS}\big)
        ~\bowtie~ \big(\proj{\sigma_1}{\PP} =
        \proj{\sigma_2}{\PP}\big)$, where $\bowtie ~\in \{\Rightarrow,
        \Leftarrow, \Leftrightarrow\}$. Then $\SSS$ is called a
        \vspace*{-0.1in}
	\begin{enumerate}
		\item \emph{generalized independent support (GIS)} of $\PP$ in
		$\varphi$ if $\bowtie$ is ~$\Leftrightarrow$
		\item \emph{upper bound support (UBS)} of $\PP$ in $\varphi$ if $\bowtie$ is ~$\Rightarrow$
		\item \emph{lower bound support (LBS)} of $\PP$ in $\varphi$ if $\bowtie$ is ~$\Leftarrow$
	\end{enumerate}
\end{definition}

Note that in the above definition, $\SSS$ need not be a subset of
$\PP$.  In fact, if $\SSS$ is restricted to be a subset of $\PP$, the
definitions of GIS and UBS coincide with that of IS
(Definition~\ref{def:is}), while LBS becomes a trivial concept (every
subset of $\PP$ is indeed an LBS of $\PP$ in $\varphi$). The following
lemma now follows immediately.
\begin{lemma}
	Let $\GG$, $\UU$ and $\LL$ be GIS, UBS and LBS, respectively, of $\PP$
	in $\varphi$.  Then $|\ProjectSatisfying{\varphi}{\LL}| \le
	|\ProjectSatisfying{\varphi}{\PP}|$ $=
	|\ProjectSatisfying{\varphi}{\GG}|$ $\le
	|\ProjectSatisfying{\varphi}{\UU}|$.
\end{lemma}

Let $\mathcal{UBS}, \mathcal{LBS}, \mathcal{GIS}$ and $\mathcal{IS}$
be the set of all UBS, LBS, GIS and IS respectively of a projection
set $\PP$ in $\varphi$.  It follows from the above definitions that
$\mathcal{IS} \subseteq \mathcal{GIS} \subseteq \mathcal{UBS}$, and
$\mathcal{GIS} \subseteq \mathcal{LBS}$. While each of the notions of
GIS, UBS and LBS are of independent interest, this paper focuses
primarily on UBS for varios reasons.  First, we found this notion to
be particularly useful in practical projected model counting.
Additionally, as the above inclusion relations show, $\mathcal{UBS}$
and $\mathcal{LBS}$ are the largest classes among $\mathcal{UBS},
\mathcal{LBS}, \mathcal{GIS}$ and $\mathcal{IS}$; hence, finding an
UBS is likely to be easier than finding a GIS.  Furthermore, the
notion of UBS continues to remain interesting (but not so for LBS, as
mentioned above) even when $\II$ is chosen to be a subset of $\PP$.

We call a UBS $\UU$ (resp. LBS $\LL$, GIS $\GG$ and IS $\II$) of $\PP$
in $\varphi$ \emph{minimal} if there is no other UBS (resp. LBS, GIS
and IS) of $\PP$ in $\varphi$ of size/cardinality strictly less than
$|\UU|$ (resp. $|\LL|$, $|\GG|$ and $|\II|$).  Minimal UBS are useful
for computing good upper bounds of projected model counts, since they
help reduce the size of random XOR clauses used in hashing-based
projected model counting.

In the remainder of this section, we first explore some interesting
theoretical properties of GIS and UBS, and then proceed to develop a
practical algorithm for computing an UBS from a given formula
$\varphi$ and projection set $\PP$.  Finally, we present an algorithm
for computing bounds of projected model counts using the UBS thus computed.
\subsection{Extremal properties of GIS and UBS}
We first show that relaxing the requirement that variables on which to
project models must be chosen from the projection set can lead to an
exponential improvement in the size of the independent/upper-bound support.
\begin{theorem}\label{thm:gis-succinct}
 For every $n > 1$, there exists a propositional formula $\varphi_n$
 on $(n-1) + \lceil\log_2 n\rceil$ variables and a projection set $\PP_n$ 
 with $|\PP_n| = n-1$ such that
 \vspace*{-0.1in}
 \begin{itemize}
 \item The smallest GIS of $\PP_n$ in
   $\varphi_n$ is of size $\lceil\log_2 n\rceil$.
 \item The smallest UBS of $\PP_n$ in
   $\varphi_n$ is of size $\lceil\log_2 n\rceil$.
 \item The smallest IS of $\PP_n$ in $\varphi_n$
   is $\PP_n$ itself, and hence of size $n-1$.
 \end{itemize}
\end{theorem}
\noindent {\bfseries \emph{Proof:}}
For notational convenience, we assume $n$ to be a power of $2$.
Consider a formula $\varphi_n$ on $(n-1) + \log_2 n$ propositional
variables $x_1, \ldots x_{n-1},$ $y_1, \ldots y_{\log_2 n}$ with $n$
satisfying assignments, say $\sigma_0, \ldots \sigma_{n-1}$, as
shown in the table below.
\begin{wrapfigure}[9]{l}{0.5\textwidth}
  \vspace*{-0.3in}
\begin{center}
  \small
\begin{tabular}{|c||c|c|c|c|c|c|c|c|}
  \hline
  & $x_1$ & $x_2$ & $\cdots$ & $x_{n-1}$ & $y_1$ & $y_2$ & $\cdots$ & $y_{\log_2 n}$ \\
  \hline
  \hline
  $\sigma_0$ & $0$ & $0$ & $\cdots$ & $0$ & $0$ & $\cdots$ & $0$ & $0$ \\
  \hline
  $\sigma_1$ & $1$ & $0$ & $\cdots$ & $0$ & $0$ & $\cdots$ & $0$ & $1$ \\
  \hline
  $\sigma_2$ & $0$ & $1$ & $\cdots$ & $0$ & $0$ & $\cdots$ & $1$ & $0$ \\
  \hline
  & $\vdots$ & $\vdots$ & $\vdots$ & $\vdots$ & $\vdots$ & $\vdots$ & $\vdots$ & $\vdots$ \\
  \hline
  $\sigma_{n-1}$ & $0$ & $0$ & $\cdots$ & $1$ & $1$ & $\cdots$ & $1$ & $1$ \\
  \hline
\end{tabular}
\end{center}
\end{wrapfigure}
\normalsize
Thus, for all $i \in \{1, \ldots n-1\}$, the values of $y_1 \ldots
y_{log_2 n}$ in $\sigma_i$ encode the number $i$ in binary (with
$y_{\log_2 n}$ being the least significant bit, and $y_1$ being the
most significant bit), the value of $x_i$ is $1$, and the values of
all other $x_j$'s are $0$.  For the special satisfying assignment
$\sigma_0$, the values of all variables are $0$.

Let $\PP_n = \{x_1, \ldots x_{n-1}\}$.  Clearly,
$|\satisfying{\varphi_n}| = |\ProjectSatisfying{\varphi_n}{\PP_n}| =
n$.  Now consider the set of variables $\GG_n = \{y_1, \ldots
y_{\log_2 n}\}$.  It is easy to verify that for every pair of
satisfying assignments $\sigma_i, \sigma_j$ of $\varphi_n$,
$\big(\proj{\sigma_i}{\GG_n} = \proj{\sigma_j}{\GG_n}\big)
\Leftrightarrow \big(\proj{\sigma_i}{\PP_n} =
\proj{\sigma_j}{\PP_n}\big)$.  Therefore, $\GG_n$ is a GIS, and hence
also a UBS, of $\PP_n$ in $\varphi_n$, and $|\GG_n| = \log_2 n$.
Indeed, specifying $y_1, \ldots y_{\log_2 n}$ completely specifies the
value of all variables for every satisfying assignment of $\varphi_n$.
Furthermore, since $|\ProjectSatisfying{\varphi_n}{\PP_n}| = n$, every
GIS and also UBS of $\PP_n$ must be of size at least $\log_2 n$.
Hence, $\GG_n$ is a smallest-sized GIS, and also a smallest-sized UBS,
of $\PP_n$ in $\varphi_n$.

Let us now find how small an independent support (IS) of $\PP_n$ in
$\varphi$ can be.  Recall that
$|\ProjectSatisfying{\varphi_n}{\PP_n}| = n$.  If possible, let there
be an IS of $\PP_n$, say $\II_n \subseteq \PP_n$, where $|\II_n| < n-1$.
Therefore, at least one variable in $\PP_n$, say $x_i$, must be absent
in $\II_n$.  Now consider the satisfying assignments $\sigma_i$ and
$\sigma_0$.  Clearly, both $\proj{\sigma_i}{\II_n}$ and
$\proj{\sigma_0}{\II_n}$ are the all-$0$ vector of size $n-1$.
Therefore, $\proj{\sigma_i}{\II_n} = \proj{\sigma_0}{\II_n}$
although $\proj{\sigma_i}{\PP_n} \neq \proj{\sigma_0}{\PP_n}$.  It
follows that $\II_n$ cannot be an IS of $\PP_n$ in $\varphi_n$.  This
implies that the smallest IS of $\PP_n$ in
$\varphi_n$ is $\PP_n$ itself, and has size $n-1$.
\qed

Observe that the smallest GIS/UBS $\GG_n$ in the above proof is
disjoint from $\PP_n$.  This suggests that it can be beneficial to
look outside the projection set when searching for a compact GIS or
UBS.  The next theorem shows that the opposite can also be true.

\begin{restatable}{theorem}{gisIsPs}\label{thm:gis-is-ps}
 For every $n > 1$, there exist propositional formulas $\varphi_n$ and
 $\psi_n$ on $(n -1) + \lceil\log_2 n\rceil$ variables and a
 projection set $\QQ_n$ with $|\QQ_n| = n-
 \lceil \log_2 n \rceil - 2$ such that
 the only GIS of $\QQ_n$ in $\varphi_n$ is $\QQ_n$, and
 the smallest UBS of $\QQ_n$ in $\psi_n$ is also $\QQ_n$.
\end{restatable}
\noindent {\bfseries \emph{Proof:}} We assume $n$ to be a power of $2$
and consider the same formula $\varphi_n$ as used in the proof of
Theorem~\ref{thm:gis-succinct}.  We use $\sigma_0, \ldots
\sigma_{n-1}$ to denote the satisfying assignments of $\varphi_n$, as
before.  Additionally, we choose $\psi_n$ to be the formula on the
same set of variables and with the same number ($n$) of satisfying
assignments, say $\widehat{\sigma_0}, \ldots \widehat{\sigma_{n-1}}$,
defined as follows:  for all $i \in \{0, \ldots n-1\}$, 
$\widehat{\sigma_i}(x_j) =
   \sigma_i(x_j)$ for $j \in \{1, \ldots n-1\}$, and
   $\widehat{\sigma_i}(y_j) = 0$
   for $j \in \{1, \ldots \log_2 n\}$. In other words, all $y_j$'s are
   set to $0$ in all satisfying assignments of $\psi_n$.
We also choose $\QQ_n = \{x_1, \ldots x_{n-1}\} \setminus \{x_{2^0},
x_{2^1}, \ldots x_{n/2}\}$.  Thus, $|\QQ_n| = (n-1) - \lceil \log_2
n\rceil - 1$ $= n - \lceil \log_2 n \rceil - 2$.

Let $\GG_n$ be a generalized independent support (GIS) of $\QQ_n$ in
$\varphi_n$. We state and prove below some auxiliary claims that help
in proving the main result.

\noindent \emph{\bfseries Claim 1:} $\forall i \in \{0 \ldots \log_2 n - 1\}$,
~~$\proj{\sigma_0}{\QQ_n} = \proj{\sigma_{2^i}}{\QQ_n}$.\\
Indeed, both $\proj{\sigma_0}{\QQ_n}$
and $\proj{\sigma_{2^i}}{\QQ_n}$ are the all-$0$ vectors of length
$|\QQ_n|$.
 
\noindent \emph{\bfseries Claim 2:} $\forall i \in \{1, \ldots \log_2 n\}$,
~~$y_i \not\in \GG_n$.\\
To see why this is so, let us express $i$ as $\log_2 n - j$, where $j
\in \{0, \ldots \log_2 n - 1\}$.  Then, it is easy to see that $y_i =
0$ in $\sigma_0$ and $y_i = 1$ in $\sigma_{2^j}$.  However, from Claim
1, we know that $\proj{\sigma_0}{\QQ_n} =
\proj{\sigma_{2^j}}{\QQ_n}$. Since $\GG_n$ is a GIS of $\QQ_n$, by properties of GIS, we cannot have $y_i \in \GG_n$.

\noindent \emph{\bfseries Claim 3:} $\forall i \in \{0, \ldots \log_2 n -1\}$,
~~$x_{2^i} \not\in \GG_n$.\\
By a similar reasoning as above, we see that $x_{2^i} = 0$ in $\sigma_0$
and $x_{2^i} = 1$ in $\sigma_{2^i}$. Therefore, since $\GG_n$ is a GIS of $\QQ_n$, from Claim 1, we cannot have $x_{2^i} \in \GG_n$.

\noindent \emph{\bfseries Claim 4:} $\forall x_j \in \QQ_n$, ~~$x_j \in
\GG_n$.\\
To see why this is so, suppose $x_j \not\in \GG_n$, and let $\QQ_n' =
\QQ_n \setminus \{x_j\}$. From Claims 2 and 3 above, we must then have
$\GG_n \subseteq \QQ_n'$.  Hence, $\QQ_n'$ must be an IS of $\QQ_n$ in
$\varphi_n$.  However, note that $\proj{\sigma_0}{\QQ_n'} =
\proj{\sigma_j}{\QQ_n'}$ although $\proj{\sigma_0}{\QQ_n} \neq
\proj{\sigma_j}{\QQ_n}$. Therefore, $\QQ_n'$ cannot be an IS of
$\QQ_n$ in $\varphi_n$ -- a contradiction! Therefore, $x_j$ must be in $\GG_n$.

\noindent Putting Claims 1 through 4 together, $\QQ_n$ is the only
GIS of $\QQ_n$ in $\varphi_n$.

Next, we show that the smallest UBS of $\QQ_n$ in $\psi_n$ is $\QQ_n$.
Let $\UU_n$ be an UBS of $\QQ_n$ in $\psi_n$ such that
there is no smaller (cardinality-wise) UBS of $\QQ_n$ in $\psi_n$.

\noindent \emph{\bfseries Claim 5:} $\forall i \in \{1, \ldots \log_2 n\}$, ~~$y_i \not\in \UU_n$.\\
This follows from the observation that $y_i = 0$ in every
$\widehat{\sigma_j}$ for $j \in \{0 \ldots n-1\}$.  Thus, if $\UU_n$
is a UBS of $\QQ_n$ in $\psi_n$ and if $y_i \in \UU_n$, then
$\UU_n \setminus \{y_i\}$ must also be a UBS of $\QQ_n$ in $\psi_n$.
This contradicts the fact that there is no other UBS of cardinality
less than $|\UU_n|$.

\noindent \emph{\bfseries Claim 6:} $\forall x_j \in \QQ_n$, ~~$x_j \in
\UU_n$.\\
To see why this is so, suppose $x_j \not\in \UU_n$ and let $\UU_n' =
\{x_1, \ldots x_{n-1}\} \setminus \{x_j\}$.  From Claim 5 above, we
must then have $\UU_n \subseteq \UU_n'$.  However, note that
$\proj{\widehat{\sigma_0}}{\UU_n'} =
\proj{\widehat{\sigma_j}}{\UU_n'}$ although
$\proj{\widehat{\sigma_0}}{\QQ_n} \neq
\proj{\widehat{\sigma_j}}{\QQ_n}$.  Therefore, $\UU_n'$ cannot be a
UBS of $\QQ_n$ in $\psi_n$.  Since $\UU_n \subseteq \UU_n'$, it
follows that $\UU_n$ cannot be a UBS of
$\QQ_n$ in $\psi_n$ either.

\noindent Putting Claims 5 and 6 together, the smallest UBS
of $\QQ_n$ in $\psi_n$ is $\QQ_n$ itself.
 \qed

Theorems~\ref{thm:gis-succinct} and \ref{thm:gis-is-ps} show that when
searching for a compact GIS or UBS, we may have to consider all
possibilities, ranging from staying within the projection set to being
completely disjoint from the projection set.

\subsection{An algorithm to compute UBS}
We now describe an algorithm to compute a UBS for a given CNF formula
$\varphi$ and projection set $\PP$.  We draw our motivation from
Padoa's theorem~\cite{P01}, which provides a necessary and sufficient
condition to determine when a variable in the support of a formula
$\varphi$ is functionally determined by the other variables in the
support.  Without loss of generality, let $\Sup{\varphi} = X = \{x_1,
x_2, \ldots x_t\}$.  We create another set of {\em fresh} variables
${X'} = \{x_1', x_2', \ldots x_t'\}$. Let $\varphi(X \mapsto {X'})$
represent the formula where every $x_i \in X$ in $\varphi$ is replaced
by $x_i' \in {X'}$.  In the following, we write $\varphi(X)$ to
clarify that the support of $\varphi$ is $X$.

\begin{lemma}[Padoa's Theorem~\cite{P01}]
  Let $\psi(X,{X'},i)$ be defined as $\varphi(X) \wedge
  \varphi(X \mapsto {X'}) \wedge \bigwedge_{j=1\atop {j\neq
      i}}^{t} (x_j \Leftrightarrow x_j') \wedge x_i \wedge \neg x_i'$.
  The variable $x_i$ is defined by $X \setminus \{x_i\}$ in the
    formula $\varphi$ iff $\psi(X,{X'},i)$ is unsatisfiable.
\end{lemma}

Padoa's theorem has been effectively used in state-of-the-art
hashing-based projected model counters such as
{\ApproxMCFour}~\cite{SGM20} to determine small independent supports
of given projection sets.  In our setting, we need to modify the
formulation since we seek to compute an upper bound support.

Given the projection set $\PP$, we first partition the support $X$ of
$\varphi$ into three sets as follows:
\begin{itemize}
\item $J$: the set of variables already determined to be in a minimal
  UBS of $\PP$ in $\varphi$
\item $Q$: the set of variables for which we do not know yet if they
  are in a minimal UBS of $\PP$ in $\varphi$
\item $D$: the set of variables that need not appear in a minimal UBS
  of $\PP$ in $\varphi$ obtained by adding elements to $J$
\end{itemize}
Initially, $J$ and $D$ are empty sets, and $Q = X$.  As the process of
computation of a minimal UBS proceeds, we wish to maintain the
invariant that $J \cup Q$ is a UBS (not necessarily minimal) of $\PP$
in $\varphi$.  Notice that this is true initially, since $X$ is
certainly a UBS of $\PP$ in $\varphi$.

Let $z$ be a variable in $Q$ for which we wish to determine if it
needs to be added to the partially computed minimal UBS $J$.  In the
following discussion, we use the notation $\varphi(J, Q\setminus
\{z\}, D, z)$ to denote $\varphi$ with its partition of variables, and
with $z$ specially identified in the partition $Q$.  Recalling the
definition of UBS from Section~\ref{sec:prelim}, we observe that if
$z$ is not part of a minimal UBS obtained by augmenting $J$, and if $J
\cup Q$ is indeed a UBS of $\PP \in \varphi$, then as long as values
other than $z$ in $J \cup Q$ are kept unchanged, the projection of a
satisfying assignment of $\varphi$ on $\PP$ must also stay unchanged.
This suggests the following check to determine if $z$ is not part of a
minimal UBS obtained by augmenting $J$.

Define $\xi(J, Q\setminus \{z\}, D, z, D', z')$ as
$\varphi(J, Q\setminus \{z\}, D, z) \wedge \varphi(J, Q\setminus
\{z\}, D', z') \wedge \bigvee_{x_i ~\in~ \mathcal{P} \cap (D \cup
  \{z\})} (x_i \not\Leftrightarrow x_i')$, where $D'$ and $z'$
represent fresh and renamed instances of variables in $D$ and $z$,
respectively.  If $\xi$ is unsatisfiable, we know that as long as the
values of variables in $J \cup (Q \setminus \{z\})$ are kept unchanged, the
projection of the satisfying assginment of $\varphi$ on $\PP$ cannot
change.  This allows us to move $z$ from the set $Q$ to the set $D$.

\begin{restatable}{theorem}{ubsCorrectness}\label{thm:check-correctness}
  If $\xi(J, Q\setminus \{z\}, D, z, D', z')$ is unsatisfiable, then
  $J \cup (Q \setminus \{z\})$ is a UBS of $\PP$ in $\varphi$.
\end{restatable}

\noindent {\bfseries \emph{Proof:}} Without loss of generality, assume
that there are no primed variables in $J, Q$ and $D$.  Let $J' =
\{x_i' \mid x_i \in J\}$, and similarly for $Q'$ and $D'$.  From the
definition of UBS in Section~\ref{sec:theory}, we know that $J \cup Q
\setminus \{z\}$ is a UBS of $\PP$ in $\varphi$ if the following
formula is valid:
\[
 \Big(\varphi(J, Q\setminus z, D, z) \wedge
 \varphi(J', Q'\setminus z', D', z') \wedge
 \bigwedge_{x_i \in (J \cup Q\setminus z)}(x_i \Leftrightarrow x_i')\Big)
 \Rightarrow \bigwedge_{x_j \in \PP}(x_j \Leftrightarrow x_j')
\]

Note that $x_i \Leftrightarrow x_i'$ is already asserted in the antecedent
of the implication for all $x_i \in (J \cup Q \setminus z)$.
Furthermore, variables in $J' \cup Q'\setminus z'$ do not appear in the
consequent of the above implication.  Since the support
of $\varphi$ is partitioned into $J, Q \setminus z, D$ and $\{z\}$, we
also have $\PP \setminus (J \cup Q\setminus z) ~=~ \PP \cap
(D \cup \{z\})$.  Putting these facts together, the condition for
$J \cup Q \setminus \{z\}$ to be a UBS reduces to validity of
\[
 \Big(\varphi(J, Q\setminus z, D, z)  \wedge 
 \varphi(J, Q\setminus z, D', z') \Big)
 \Rightarrow \bigwedge_{x_j \in \PP \cap (D \cup \{z\})}(x_j \Leftrightarrow x_j')
\]

In other words, $J \cup Q\setminus \{z\}$ is a UBS of $\PP$ in $\varphi$
if $\varphi(J, Q\setminus z, D, z) \wedge \varphi(J, Q\setminus z, D', z')
 \wedge \bigvee_{x_j \in \PP \cap (D \cup \{z\})}(x_j \not\Leftrightarrow x_j')$
 is unsatisfiable.

\qed

The above check suggests a simple algorithm for computing a minimal
UBS.  We present the pseudocode of our algorithm for computing UBS
below.
\begin{algorithm}[h]
\small
\begin{algorithmic}[1]
\State $J \gets \emptyset; Q \gets \Sup{\varphi}; D \gets \emptyset$;
\Repeat
 \State $z \gets \ChooseNextVar(Q)$;
 \State $\xi \gets \left(\begin{array}{c} \varphi(J, Q\setminus z, D, z) ~\wedge~ \varphi(J, Q\setminus z, D', z') ~\wedge \\\bigvee_{x_i ~\in~ \PP \cap (D \cup \{z\})} \neg (x_i \Leftrightarrow x_i')   \end{array}\right)$;
 \If {$\xi$ is  \UNSAT}
     \State $D \gets D \cup \{z\}$;
 \Else
     \State $J \gets J \cup \{z\}$;
 \EndIf
 \State $Q \gets Q \setminus \{z\}$;
\Until{$Q$ is $\emptyset$};
\State \Return $J$;
\end{algorithmic}

\caption{{\findubs}$(\varphi, \PP)$}
\label{alg:findubs}
\end{algorithm}

After initializing $J$, $Q$ and $D$, {\findubs} chooses a variable $z
\in Q$ and checks if the formula $\xi$ in
Theorem~\ref{thm:check-correctness} is unsatisfiable.  If so, it adds
$z$ to $D$ and removes it from $Q$.  Otherwise, it adds $z$ to $J$.
The algorithm terminates when $Q$ becomes empty.  On termination, $J$
gives a minimal UBS of $\PP$ in $\varphi$.  The strategy for choosing
the next $z$ from $Q$, implemented by sub-routine $\ChooseNextVar$,
clearly affects the quality of UBS obtained from this algorithm.  We
require that $\ChooseNextVar(Q)$ return a variable from $Q$ as long as
$Q \neq \emptyset$.  Choosing $z$ from within $\PP$ gives a UBS that
is the same as an IS of $\PP$ in $\varphi$.  In our experiments, we
therefore bias the choice of $z$ to favour those not in $\PP$ when
selecting variables from $Q$.

We now state some key properties of Algorithm {\findubs}. 

\begin{restatable}{lemma}{findubsInvariant}\label{lem:findubs-invariant}
The following invariant holds at the loop head (line 2) of 
Algorithm~\ref{alg:findubs}:
There exists a minimal UBS $\UU^*$ of $\PP$ in $\varphi$ such that
$J \subseteq \UU^* \subseteq J \cup Q$.
\end{restatable}
\begin{proof}
From the initialization of $J$ and $Q$ in line 1, we know that when
Algorithm~\ref{alg:findubs} reaches line 2 for the first time, $Q
= \Sup{\varphi}$ and $J = \emptyset$.  Since every minimal UBS of
$\PP$ in $\varphi$ must be contained in $\Sup{\varphi}$ and must
contain $\emptyset$, it follows that the invariant is satisfied when
the algorithm reaches line 2 for the first time.

Let us now inductively hypothesize that the invariant holds at the
start of $k^{th}$ iteration of the repeat-until loop.  There are two
cases to consider in the inductive step.
\begin{itemize}
	\item If the formula $\xi$ (see line 4 of Algorithm~\ref{alg:findubs})
	is $\UNSAT$, then $J$ stays unchanged in the $k^{th}$ iteration.
	Furthermore, from Theorem~\ref{thm:check-correctness}, we know that
	$J \cup (Q\setminus \{z\})$ is an UBS of $\PP$ in $\varphi$.  This
	implies there is at least one minimal UBS that is a subset of $J (\cup
	Q\setminus \{z\})$.  Since $z$ is removed from $Q$ in line 9, our
	proof would be complete if there exists a minimal UBS $\UU^*$ such
	that $J \subseteq \UU^* \subseteq (J \cup Q \setminus \{z\})$.
	Suppose, if possible, there is no such minimal UBS $\UU^*$.  From the
	inductive hypothesis, we also know that there exists at least one
	minimal UBS, say $\widehat{\UU^*}$, such that
	$J \subseteq \widehat{\UU^*} \subseteq (J \cup Q)$. The above
	assertions imply that $z$ must be present in \emph{every} minimal UBS
	$\widehat{\UU^*}$ that satisfies
	$J \subseteq \widehat{\UU^*} \subseteq (J \cup Q)$.  This, in turn,
	implies that $\varphi(J, Q\setminus z, D, z) \wedge \varphi(J,
	Q\setminus z, D', z') \not\Rightarrow \bigwedge_{x_i ~\in~ \PP \cap
		(D \cup \{z\})}(x_i \Leftrightarrow x_i')$.  This contradicts the
	assumption that the formula $\xi$ is $\UNSAT$.  Hence, our assumption
	must be incorrect, and there exists a minimal UBS $\UU^*$ such that
	$J \subseteq \UU^* \subseteq (J \cup Q \setminus \{z\})$.
	
	Therefore, the invariant holds at the loop head at the start of the
	$k+1^{st}$ iteration of the loop.
	
	\item Let $\sigma$ be a satisfying assignment of $\xi$. In this case,
	the variable $z$ is removed from $Q$ and added to $J$ in lines 8 and
	9; hence $J \cup Q$ stays unchanged.
	
	The assignment of values to variables in $\sigma$ yields two
	satisfying assignments of $\varphi$ that agree on the values of
	variables in $J \cup (Q \setminus {z})$, and yet the projections of
	the satisfying assignments on $\PP$ are different.  Therefore, $J \cup
	(Q \setminus \{z\})$ is not a UBS of $\PP$ in $\varphi$.  Since every
	superset of a UBS is also a UBS (easy consequence of the definition of
	UBS), this implies that there is no minimal UBS that is contained in
	$J \cup (Q \setminus \{z\})$.  However, we know from the inductive
	hypothesis that there is at least one minimal UBS, say $\UU^*$, such
	that $J \subseteq \UU^* \subseteq J \cup Q$.  The above two assertions
	imply that $z \in \UU^*$.  Hence,
	$(J \cup \{z\}) \subseteq \UU^* \subseteq (J \cup \{z\}) \cup
	(Q \setminus \{z\})$.
	
	Once again, the invariant holds at the loop head at the start of the
	$k+1^{st}$ iteration of the loop.
\end{itemize}

\end{proof}

\begin{restatable}{theorem}{findubsGivesMinimalUBS}\label{thm:findubs-minimal}
Algorithm~\ref{alg:findubs}, when invoked on $\varphi$ and $\PP$,
terminates and computes a minimal UBS of $\PP$ in $\varphi$.
\end{restatable}
\begin{proof}
	 In line 1 of Algorithm~\ref{alg:findubs},
the set $Q$ is initialized with $\Sup{\varphi}$, which has finitely
many variables.  Subsequently, in each iteration of the repeat-until
loop in lines 2--10, exactly one variable is removed from $Q$, and no
other variable is added to $Q$.  Therefore, $Q$ must become empty
after $|\Sup{\varphi}|$ iterations of the loop, leading to termination
of Algorithm~\ref{alg:findubs}.

Let $J^*$ and $Q^*$ denote the values of $J$ and $Q$, respectively, on
termination.  From Lemma~\ref{lem:findubs-invariant}, we know that
there is a minimal UBS $\UU^*$ of $\PP$ in $\varphi$ such that
$J^* \subseteq \UU^* \subseteq J^* \cup Q^*$.  However, $Q^*
= \emptyset$ on termination of the algorithm.  Hence $J^* = \UU^*$,
and Algorithm~\ref{alg:findubs} terminates with a minimal UBS of $\PP$
in $\varphi$.
\end{proof}

The overall algorithm for computing an upper bound of the
projected model count of a CNF formula using {\UBS} is shown in
Algorithm~\ref{alg:arjunubs}.  This algorithm takes two timeout
parameters, $\tau_{\mathrm{pre}}$ and $\tau_{\mathrm{count}}$. These
are used to limit the times taken for computing a UBS using algorithm
{\findubs}, and for computing a projected model count of $\varphi$ on
the computed UBS, respectively.

\begin{algorithm}[h]
\small
\caption{{\arjunubs}$(\varphi, \PP, \varepsilon, \delta, \tau_{\mathrm{pre}}, \tau_{\mathrm{count}})$}
\label{alg:arjunubs}

\begin{algorithmic}[1]
\State $\mathcal{U} \gets$ {\findubs}$(\varphi, \PP)$ with timeout $\tau_{\mathrm{pre}}$;
\If {call to {\findubs} times out}
 \State $\mathcal{U} \gets \PP$;
\EndIf
\State \Return $\mathsf{ComputeCount}(\varphi, \mathcal{U}, \varepsilon, \delta)$
\end{algorithmic}
\end{algorithm}

\begin{restatable}{theorem}{ubscounterCorrect}\label{thm:ubs-counter-correct}
Given a CNF formula $\varphi$, a projection set $\PP$, parameters
$\varepsilon~(> 0)$ and $\delta~(0 < \delta < 1)$, and given access to a PAC counter, $\mathsf{ComputeCount}$, suppose
Algorithm~\ref{alg:arjunubs} returns a count $c$.  Then for every
choice of $\ChooseNextVar$, we have
$\Prob{|\ProjectSatisfying{\varphi}{\PP}| \le (1 + \varepsilon)\cdot
  c} \ge 1 - \delta$.
\end{restatable}

\begin{proof} From Theorem~\ref{thm:findubs-minimal},
we know that when Algorithm~\ref{alg:arjunubs} reaches line 4, the set
$\UU$ is a UBS of $\PP$ in $\varphi$.  Hence
$|\ProjectSatisfying{\varphi}{\PP}| \leq
|\ProjectSatisfying{\varphi}{\UU}|$.  Furthermore, since $\mathsf{ComputeCount}$ is a PAC counter, we have
$\Prob{\frac{|\ProjectSatisfying{\varphi}{\UU}|}{1+\varepsilon} \leq c \leq (1+\varepsilon)\cdot|\ProjectSatisfying{\varphi}{\UU}|} $ $ \ge 1- \delta$.
The theorem now follows easily from the above two observations.
\end{proof}
Theorem~\ref{thm:ubs-counter-correct} provides the weakest guarantee
for Algorithm {\arjunubs}, ignoring the details of sub-routine
$\ChooseNextVar$. In practice, the specifics of $\ChooseNextVar$ can
be factored in to strengthen the guarantee, including providing
PAC-style guarantees in the extreme case if $\ChooseNextVar$ always
chooses variables from the projection set $\PP$.  A more detailed
analysis of {\arjunubs}, taking into account the specifics of
$\ChooseNextVar$ is beyond the scope of this paper.

\section{Experimental Evaluation}\label{sec:evaluation}

To evaluate practical efficiency of {\arjunubs}, we implemented a prototype in C++. Our prototype implementation builds on {\arjun}~\cite{SM21}. %
To demonstrate practical impact of computation of {\UBS} on model counting, we first compute {\UBS} and pass the UBS as a projection set along with the formula $\varphi$ to the state of the art hashing-based approximate counter {\ApproxMCFour}~\cite{SGM20}. 

 We chose {\ApproxMCFour} because it is a highly competitive projected model counter, a version of which won the 2020 model counting competition in the projected counting track. In this context, it is worth remarking that the 2021 competition sought to focus on the exact counters (which do not rely on hashing-based methods) and consequently, changed required $\varepsilon = 0.1$ to $\varepsilon = 0.01$. Since the complexity of approximate techniques has $\varepsilon^{-2}$ dependence, this led to nearly 100$\times$ penalty in runtime for the approximate techniques.  Even then, {\ApproxMCFour}-based entry achieved 3rd place. As described during the competitive event presentation at the SAT 2021, had $\varepsilon$ been set to $0.05$, the {\ApproxMCFour}-based entry would have won the competition. All prior applications and benchmarking for approximation techniques have been presented with $\varepsilon=0.8$ in the literature, and we continue to use the same value of $\varepsilon$ in this work.

We use {\UBS}+{\ApproxMCFour} to denote the case when {\ApproxMCFour} is invoked with the computed {\UBS} as the projection set while we use {\IS}+{\ApproxMCFour} to refer to the version of {\ApproxMCFour} invoked with {\IS} as the projection set.

	\paragraph{Benchmarks}
	Our benchmark suite consisted of 2632 instances, which are categorized into four categories:
	BNN, Circuit, QBF-exist and QBF-circuit. 
	BNN benchmarks are adapted from \cite{BSSM+19}. Each instance contains CNF encoding of binarized neural networks (BNN) and constraints from properties of interest such as robustness, cardinality, and parity. The projection set $\mathcal{P}$ is set to variables for a chosen layer. 
	The class `Circuit' refers to instances from~\cite{CMV14}, which encode circuits arising from ISCAS benchmarks conjuncted with random parity constraints imposed on output variables. The projection set, as set by authors in~\cite{CMV14}, corresponds to output variables. The benchmarks with `QBF' are based on instances from Prenex-2QBF track of QBFEval-17\footnote{http://www.qbflib.org/qbfeval17.php}, QBFEval-18\footnote{http://www.qbflib.org/qbfeval18.php}, disjunctive~\cite{ACJS17}, arithmetic~\cite{TV17} and factorization~\cite{ACJS17}. 
	Each `QBF-exist' benchmark is a CNF formula transformed from a QBF instance. We remove quantifiers for the (2-)QBF instance and set the projection set to existentially quantified variables. 
	The class `QBF-circuit' refers to circuits synthesized using the state of the art functional synthesis tool, Manthan~\cite{GRM20}. The projection set is set to output variables.

	Experiments were conducted on a high-performance computer cluster, each node consisting of 2xE5-2690v3 CPUs with 2x12 real cores and 96GB of RAM. The model counter with per preprocessing technique on a particular benchmark runs on a single core. We set the time limit as 5000 seconds for preprocessing and counting respectively and the memory limit as 4GB. The maximal number of conflicts is 100k for each candidate variable during preprocessing.  To compare runtime performance, we employ PAR-2 scores, which is a standard in SAT community, in the context of SAT competition. Each benchmark contributes a score that is the number of seconds used by the corresponding tool to successfully finish the execution or in case of a timeout or memory out, twice the timeout in seconds. We then calculate the average score for all benchmarks, giving PAR-2.

	We sought to answer the following research questions: 
	\begin{description}
		
		\item[RQ 1] Does the usage of {\UBS} enable {\ApproxMCFour} to solve more benchmarks in comparison to the usage of {\IS} ? 
		\item[RQ 2] How does the quality of counts computed by {\UBS}+{\ApproxMCFour} vary in comparison to {\IS}+{\ApproxMCFour}?
		\item [RQ 3] How does the runtime behavior of {\UBS}+{\ApproxMCFour} compare with that of {\IS}+{\ApproxMCFour}?
	\end{description}

	\paragraph{Summary}
	In summary, {\UBS}+{\ApproxMCFour} solves 208 more instances than {\IS}+{\ApproxMCFour}. Furthermore, while computation of {\UBS} takes 777 more seconds, the PAR-2 score of {\UBS}+{\ApproxMCFour} is 817 seconds less than that of {\IS}+{\ApproxMCFour}. Finally, for all the instances where both {\UBS}+{\ApproxMCFour} and {\IS}+{\ApproxMCFour} terminated the geometric mean of log-ratio of counts returned by {\IS}+{\ApproxMCFour} and {\UBS}+{\ApproxMCFour} is 1.32, therefore, indicating that {\UBS}+{\ApproxMCFour} provides accurate estimates and therefore, can be used as substitute for {\IS}+{\ApproxMCFour} in the context of applications that primarily care about upper bounds. 
	
	In this context, it is worth highlighting that since there has been considerable effort in recent years in the optimizing computation of {\IS}, one would expect that further engineering efforts would lead to even more saving in runtime in the usage of {\UBS}.

	\subsection{Number of Solved Benchmarks}

	\begin{table}[htb]
	\centering
	\setlength{\tabcolsep}{9pt}
	\begin{tabular}{c c c c c}
		\toprule
		Benchmarks & Total & VBS & {\IS}+{\ApproxMCFour} & {\UBS}+{\ApproxMCFour}  \\
		\midrule
		BNN & 1224 & 868 & 823 & 823   \\
		Circuit & 522 & 455 & 407 & \textbf{435}  \\
		QBF-exist & 607 & 314 & 156 & \textbf{291}   \\
		QBF-circuit & 279 & 152 & 100 & \textbf{145}  \\
		\bottomrule
	\end{tabular}
	\caption{The number of solved benchmarks. }
	\label{tab: number solved}
\end{table}
	
	Table~\ref{tab: number solved} compares the number of solved benchmarks by {\IS}+{\ApproxMCFour} and {\UBS}+{\ApproxMCFour}. Observe that the usage of {\UBS} enables {\ApproxMCFour} to solve 435, 291, and 145 instances on Circuit, QBF-exist, and QBF-circuit benchmark sets respectively while the usage of {\IS}+{\ApproxMCFour} solved 407, 156 and 100 instances. In particular, {\UBS}+{\ApproxMCFour} solved almost twice as many instances on QBF-exist benchmarks. 
	
	The practical adoption of tools for NP-hard problems often rely on portfolio solvers.  Therefore, from the perspective of practice, one is often interested in evaluating the impact of a new  technique to the portfolio of existing state of the art. 
	To this end, we often focus on Virtual Best Solver (VBS), which can be viewed as an ideal portfolio.  An instance is considered to be solved by VBS if is solved by at least one solver in the portfolio; that is, VBS is at least as powerful as each solver in the portfolio. Observe that on BNN benchmarks, while {\UBS}+{\ApproxMCFour} and {\IS}+{\ApproxMCFour} solved the same number of instances but VBS solves 45 more instances as there were 45 instances that were solved by {\UBS}+{\ApproxMCFour} but not {\IS}+{\ApproxMCFour}.

	\subsection{Time Analysis}

	To analyze the runtime behavior, we separate the preprocessing time (computation of {\UBS} and {\IS}) and the time taken by {\ApproxMCFour}. Table \ref{tab: time} reports the mean of preprocessing time over benchmarks and the PAR-2 score for counting time. 
	The usage of {\UBS} reduces the PAR-2 score for counting from  from 3680, 2206, 7493, and 6479 to 3607, 1766, 5238, and 4829 on four benchmark sets, respectively. Remarkably, {\UBS} reduces PAR-2 score by over 2000 seconds on QBF-exist benchmarks and over 1000 seconds on QBF-circuit, which is a significant improvement. 

	\begin{table}[htb]
		\centering
		\begin{tabular}{c c c c c}
			\toprule
			&\multicolumn{2}{c}{Preprocessing time} & \multicolumn{2}{c}{PAR-2 score of counting time} \\
			\cmidrule(lr){2-3}\cmidrule(lr){4-5}
			Benchmarks  & 	\hphantom{0000}{\IS} (s)\hphantom{0000} & \hphantom{0000}{\UBS} (s)\hphantom{0000} & \hphantom{0000}{\IS} (s)\hphantom{0000} & \hphantom{0000}{\UBS} (s)\hphantom{0000} \\
			\midrule
			BNN & {2518}& 2533 & 3680 &\textbf{3607}\\
			Circuit & {\hphantom{0}229}& \hphantom{0}680 & 2206 &\textbf{1766}\\
			QBF-exist& \hphantom{00}{70}& 2155 & 7493 &\textbf{5238}\\
			QBF-circuit & {\hphantom{0}653}& 2541 & 6479 &\textbf{4829} \\
			\bottomrule
		\end{tabular}
		
		\caption{The mean of preprocessing time and PAR-2 score of counting time}
		\label{tab: time}
	\end{table}

	Observe that the mean time taken by {\UBS} is higher than that of {\IS} across all four benchmark classes. Such an observation may lead one to wonder whether savings due to {\UBS} are indeed useful; in particular, one may argue that if the total time of {\IS}+{\ApproxMCFour} is set to 10,000 seconds so that the time remaining after {\IS} computation can be used by {\ApproxMCFour}. We observe that even in such a case, {\IS}+{\ApproxMCFour} is able to solve only four more instances than Table~\ref{tab: number solved}.\footnote{Exclude BNN benchmarks because the computation of {\UBS} and {\IS} takes similar time on BNN benchmarks.} To further emphasize, {\UBS}+{\ApproxMCFour} where {\ApproxMCFour} is allowed a timeout of 5000 seconds can still solve more instance than {\IS}+{\ApproxMCFour} where {\ApproxMCFour} is allowed a timeout of $10,000 - t_{\IS}$ where $t_{\IS}$ is time taken to compute {\IS} with a timeout of 5000 seconds.

	\paragraph{Detailed Runtime Analysis}
		\begin{table*}[htb]
		\centering
		\begin{adjustbox}{width=1\textwidth}
			\begin{tabular}{c r r r r c r c c}
				\toprule
				&& & 	 
				\multicolumn{3}{c}{{\IS}+{{\ApproxMCFour}}} &  \multicolumn{3}{c}{{\UBS}+{\ApproxMCFour}} \\
				\cmidrule(lr){4-6}\cmidrule(lr){7-9}
				
				Benchmarks & $|X|$ & $|\ensuremath{\mathcal{P}}|$ & 	$|\text{{\IS}}|$ & {Time (s)} &  Count & 	$|\text{{\UBS}}|$ &  Time (s) & Count \\
				\midrule
				amba2c7n.sat & 1380 & 1345 & 				313 & \hphantom{}0.24+2853 & $50*2^{65\hphantom{00}}$ & 					73 & \hphantom{00}17+1\hphantom{000}	& $63*2^{67\hphantom{0}}$ \\
				bobtuint31neg & 1634 & 1205 & 				678 & \hphantom{}0.37+5000 & $-$ &											417 & \hphantom{0}148+16\hphantom{00} & $64*2^{411}$ \\
				ly2-25-bnn\_32-bit-5-id-11	& 131 & 32 & 	32 & 1313+3416 & $94*2^{9\hphantom{000}}$ &								59 & 2113+1034 &	$63*2^{10\hphantom{0}}$ \\
				ly3-25-bnn\_32-bit-5-id-10 & 131 & 32 & 		32 & 1389+5000 & $-$ & 														61 & 2319+841\hphantom{0} & $60*2^{9	\hphantom{00}}$ \\
				floor128	& 891 & 879 &					254 & \hphantom{}0.07+5000& $-$	&											256 & \hphantom{000}9+6\hphantom{000} & $64*2^{250}$ \\
				s15850\_10\_10.cnf & 10985 & 684 & 			605 & \hphantom{}0.50+5000	& $-$ & 											600 & \hphantom{00}41+2070 & $50*2^{566}$ \\
				arbiter\_10\_5 & 23533 & 129 & 				118 & \hphantom{}0.71+4\hphantom{000} & $64*2^{112\hphantom{0}}$ &		302 & \hphantom{000}7+5000 & $-$	\\
				cdiv\_10\_5 & 101705 & 128 & 				60 & \hphantom{0}102+50\hphantom{00} & $72*2^{50\hphantom{00}}$ &		$-$ & 5000+5000 & $-$ \\
				rankfunc59\_signed\_64 & 5140 & 4505 & 		1735 & \hphantom{000}3+274\hphantom{0} & $43*2^{1727}$ & 				$-$ & 5000+5000 & $-$ \\
				
				\bottomrule
			\end{tabular}
		\end{adjustbox}
		
		\caption{Performance comparison of {\UBS} vs. {\IS}. The runtime is reported in seconds and ``$-$' in a column reports timeout after 5000 seconds.}
		\label{tab: case study}
	\end{table*}

	Table~\ref{tab: case study} presents the results over a subset of benchmarks. Column 1 of the table gives the benchmark name, while columns 2 and 3 list the size of support $X$ and the size of projection set $\mathcal{P}$, respectively. Columns 4-6 list the size of computed {{\IS}}, runtime of {\IS}+{\ApproxMCFour}, and model count over {\IS} while columns 7-9 correspond to {\UBS}. Note that the time is represented in the form $t_p + t_c$ where $t_p$ refers to the time taken by {\IS} (resp. {\UBS}) and $t_c$ refers to the time taken by {\ApproxMCFour}. We use `$-$' in column 6 (resp. column 9) for the cases where {\IS}+{\ApproxMCFour}  (resp. {\UBS}+{\ApproxMCFour}) times out.  
	For example, on benchmark amba2c7n.sat, the computation of {\IS} takes 0.24 seconds and returns an independent support of size 313 variables while the computation of {\UBS} takes 17 seconds and the computed {\UBS} has 73 variables. Furthermore,  {\ApproxMCFour} when supplied with {\IS} takes 2853 seconds while {\ApproxMCFour} with {\UBS} takes only one second. The model count returned by {\ApproxMCFour} in case of {\IS} is  $50*2^{65}$ while the count returned by {\ApproxMCFour} in case of {\UBS}  is $63*2^{67}$. 
	
	The benchmark set was chosen to showcase different behaviors of interest: First, we observe that the  smaller size of {\UBS} for amba2c7n.sat leads {\UBS}+{\ApproxMCFour} while {\IS}+{\ApproxMCFour} times out. It is, however, worth emphasizing that the size of {\UBS} and {\IS} is not the only factor. To this end, observe that for the two benchmarks arising from BNN, represented in the third and fourth row, even though the size of {\UBS} is large, the runtime of {\ApproxMCFour} is still improved. 
Furthermore, in comparison to {\IS}, our implementation for {\UBS} did not explore engineering optimization, which shows how {\UBS} computation times out in the presence of the large size of support. Therefore, an important direction of future research is to further optimize the computation of {\UBS} to fully unlock the potential of {\UBS}.

	\subsection{Quality of Upper Bound}
	
	To evaluate the quality of upper bound, we compare the counts computed by {\UBS}+{\ApproxMCFour} with those of {\IS}+{\ApproxMCFour} for 1376 instances where both {\IS}+{\ApproxMCFour} and {\UBS}+{\ApproxMCFour} terminated.  Suppose $C_\text{IS}$ and $C_\text{UBS}$ denote the model count on {\IS} and {\UBS} respectively. The error $\mathsf{Error}$ is computed by $\mathsf{Error} = \log_2 C_\text{UBS}-\log_2 C_\text{IS}$.  Figure \ref{fig: count error} shows $\mathsf{Error}$ distribution over benchmarks. A point $(x, y)$ represents $\mathsf{Error} \le y$ on the first $x$ benchmarks. For example, the point $(1000, 2.2)$ means the $\mathsf{Error}$ is not more than 2.2 on one thousand benchmarks. Overall, the geometric mean of $\mathsf{Error}$ is just 1.32.
	Furthermore, for more than 67\% benchmarks the $\mathsf{Error}$ is less than one, and for 81\% benchmarks,  the the $\mathsf{Error}$ is less than five while on only 11\% benchmarks the $\mathsf{Error}$ is larger than ten. To put the significance of  $\mathsf{Error}$ margin in the context, we refer to the recent survey~\cite{APM21} comparing several partition function estimation  techniques, wherein a method with  $\mathsf{Error}$ less than 5 is labeled as a {\em reliable method}. It is known that partition function estimate reduces to model counting, and the best performing technique identified in that study relies on model counting.

	\begin{figure}[h!]
		\centering
		\includegraphics[scale=0.9]{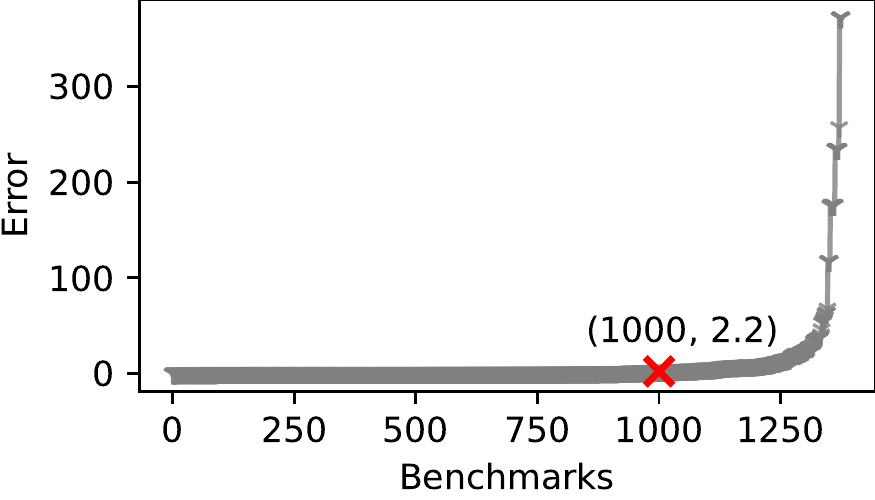}
		\caption{Error of upper bound.}
		\label{fig: count error}
	\end{figure}

\section{Conclusion}\label{sec:conclusion}
In this work, we introduced the notion of Upper Bound Support ({\UBS}), which generalizes the well-known notion of independent support. We then observed that the usage of {\UBS} for generation of XOR constraints in the context of approximate projected model counting leads to the computation of upper bound of projected model counts. Our empirical analysis demonstrates that {\UBS}+{\ApproxMC} leads to significant runtime improvement in terms of the number of instances solved as well as the PAR-2 score. Since identification of the importance of {\IS} in the context of counting led to follow-up work focused on efficient computation of {\IS}, we hope our work will excite the community to work on efficient computation of {\UBS}. 
\section*{Acknowledgments}
This work was supported in part by National Research Foun-
dation Singapore under its NRF Fellowship Programme[NRF-
NRFFAI1-2019-0004 ] and AI Singapore Programme [AISG-
RP-2018-005], and NUS ODPRT Grant [R-252-000-685-13].
The computational work for this article was performed on
resources of the National Supercomputing Centre, Singapore
(https://www.nscc.sg).


\begin{thebibliography}{10}
\providecommand{\url}[1]{\texttt{#1}}
\providecommand{\urlprefix}{URL }
\providecommand{\doi}[1]{https://doi.org/#1}

\bibitem{AHT18}
Achlioptas, D., Hammoudeh, Z., Theodoropoulos, P.: Fast and flexible
  probabilistic model counting. In: Proc. of SAT. pp. 148--164 (2018)

\bibitem{AT17}
Achlioptas, D., Theodoropoulos, P.: Probabilistic model counting with short
  xors. In: Proc. of SAT. pp. 3--19 (2017)

\bibitem{APM21}
Agrawal, D., Pote, Y., Meel, K.S.: Partition function estimation: A
  quantitative study. In: Proc. of IJCAI (8 2021)

\bibitem{ACJS17}
Akshay, S., Chakraborty, S., John, A.K., Shah, S.: Towards parallel boolean
  functional synthesis. In: Proc. of TACAS. pp. 337--353 (2017)

\bibitem{AD16}
Asteris, M., Dimakis, A.G.: Ldpc codes for discrete integration. Technical
  report, UT Austin (2016)

\bibitem{BSSM+19}
Baluta, T., Shen, S., Shine, S., Meel, K.S., Saxena, P.: Quantitative
  verification of neural networks and its security applications. In: Proc. of
  CCS (11 2019)

\bibitem{BEHLMQ18}
Biondi, F., Enescu, M., Heuser, A., Legay, A., Meel, K.S., Quilbeuf, J.:
  Scalable approximation of quantitative information flow in programs. In:
  Proc. of VMCAI (1 2018)

\bibitem{CMV13a}
Chakraborty, S., Meel, K.S., Vardi, M.Y.: A scalable and nearly uniform
  generator of sat witnesses. In: Proc. of CAV. pp. 608--622 (7 2013)

\bibitem{CMV13b}
Chakraborty, S., Meel, K.S., Vardi, M.Y.: A scalable approximate model counter.
  In: Proc. of CP. pp. 200--216 (9 2013)

\bibitem{CMV14}
Chakraborty, S., Meel, K.S., Vardi, M.Y.: Balancing scalability and uniformity
  in sat-witness generator. In: Proc. of DAC. pp. 60:1--60:6 (6 2014)

\bibitem{CMV16}
Chakraborty, S., Meel, K.S., Vardi, M.Y.: Algorithmic improvements in
  approximate counting for probabilistic inference: From linear to logarithmic
  {SAT} calls. In: Proc. of IJCAI (2016)

\bibitem{CD05}
Chavira, M., Darwiche, A.: Compiling bayesian networks with local structure.
  In: IJCAI. vol.~5, pp. 1306--1312 (2005)

\bibitem{DMPV17}
Duenas-Osorio, L., Meel, K.S., Paredes, R., Vardi, M.Y.: Counting-based
  reliability estimation for power-transmission grids. In: Proc. of AAAI (2
  2017)

\bibitem{EGSS14}
Ermon, S., Gomes, C., Sabharwal, A., Selman, B.: Low-density parity constraints
  for hashing-based discrete integration. In: Proc. of ICML. Proceedings of
  Machine Learning Research, vol.~32, pp. 271--279 (22--24 Jun 2014),
  \url{https://proceedings.mlr.press/v32/ermon14.html}

\bibitem{EGSS13b}
Ermon, S., Gomes, C.P., Sabharwal, A., Selman, B.: Taming the curse of
  dimensionality: Discrete integration by hashing and optimization. In: Proc.
  of ICML. {JMLR} Workshop and Conference Proceedings, vol.~28, pp. 334--342 (6
  2013), \url{http://proceedings.mlr.press/v28/ermon13.html}

\bibitem{GRM20}
Golia, P., Roy, S., Meel, K.S.: Manthan: A data-driven approach for boolean
  function synthesis. In: Proc. of CAV (7 2020)

\bibitem{GHSS07a}
Gomes, C., Hoffmann, J., Sabharwal, A., Selman, B.: Short xors for model
  counting: From theory to practice. In: Proc. of SAT (2007)

\bibitem{GHSS07b}
Gomes, C., Hoffmann, J., Sabharwal, A., Selman, B.: From sampling to model
  counting. pp. 2293--2299 (01 2007)

\bibitem{GSS06}
Gomes, C.P., Sabharwal, A., Selman, B.: Model counting: A new strategy for
  obtaining good bounds. In: Proc. of AAAI. AAAI'06, vol.~1, p. 54–61 (2006)

\bibitem{IMMV16}
Ivrii, A., Malik, S., Meel, K.S., Vardi, M.Y.: On computing minimal independent
  support and its applications to sampling and counting. Constraints
  \textbf{21}(1) (9 2016)

\bibitem{KSS08}
Kroc, L., Sabharwal, A., Selman, B.: Leveraging belief propagation, backtrack
  search, and statistics for model counting. In: Proc. of CPAIOR. p. 127–141.
  CPAIOR'08 (2008)

\bibitem{LLM16}
Lagniez, J.M., Lonca, E., Marquis, P.: Improving model counting by leveraging
  definability. In: IJCAI. pp. 751--757 (2016)

\bibitem{AM20}
Meel, K.S., Akshay, S.: Sparse hashing for scalable approximate model counting:
  Theory and practice. In: Proc. of LICS (7 2020)

\bibitem{MVCF+16}
Meel, K.S., Vardi, M.Y., Chakraborty, S., Fremont, D.J., Seshia, S.A., Fried,
  D., Ivrii, A., Malik, S.: Constrained sampling and counting: Universal
  hashing meets sat solving. In: Proc. of Workshop on BNP (2016)

\bibitem{P01}
Padoa, A.: Essai d’une théorie algébrique des nombres entiers, précédé
  d’une introduction logique à une théorie déductive quelconque.
  Bibliothèque du Congrès International de Philosophie  \textbf{3}, ~309
  (1901)

\bibitem{R96}
Roth, D.: On the hardness of approximate reasoning. Artificial Intelligence
  \textbf{82}(1),  273--302 (1996).
  \doi{https://doi.org/10.1016/0004-3702(94)00092-1},
  \url{https://www.sciencedirect.com/science/article/pii/0004370294000921}

\bibitem{SBK05}
Sang, T., Bearne, P., Kautz, H.: Performing bayesian inference by weighted
  model counting. In: Proc. of AAAI. AAAI'05, vol.~1, p. 475–481 (2005)

\bibitem{SRSM19}
Sharma, S., Roy, S., Soos, M., Meel, K.S.: Ganak: A scalable probabilistic
  exact model counter. In: IJCAI. vol.~19, pp. 1169--1176 (2019)

\bibitem{SGM20}
Soos, M., Gocht, S., Meel, K.S.: Tinted, detached, and lazy cnf-xor solving and
  its applications to counting and sampling. In: Proc. of CAV (7 2020)

\bibitem{SM19}
Soos, M., Meel, K.S.: Bird: Engineering an efficient cnf-xor sat solver and its
  applications to approximate model counting. In: Proc. of AAAI (1 2019)

\bibitem{SM21}
Soos, M., Meel, K.S.: Arjun: An efficient independent support computation
  technique and its applications to counting and sampling. CoRR  (2021)

\bibitem{S83}
Stockmeyer, L.: The complexity of approximate counting. In: Proc. of STOC. p.
  118–126. STOC '83 (1983). \doi{10.1145/800061.808740},
  \url{https://doi.org/10.1145/800061.808740}

\bibitem{TV17}
Tabajara, L.M., Vardi, M.Y.: Factored boolean functional synthesis. In: Proc.
  of FMCAD. p. 124–131. FMCAD '17 (2017)

\bibitem{TW21}
Teuber, S., Weigl, A.: Quantifying software reliability via model-counting. In:
  International Conference on Quantitative Evaluation of Systems. pp. 59--79.
  Springer (2021)

\bibitem{V79a}
Valiant, L.G.: The complexity of computing the permanent. Theoretical Computer
  Science  \textbf{8}(2),  189--201 (1979)

\bibitem{V79}
Valiant, L.G.: The complexity of enumeration and reliability problems. SIAM
  Journal on Computing  \textbf{8}(3),  410--421 (1979). \doi{10.1137/0208032},
  \url{https://doi.org/10.1137/0208032}

\bibitem{ZCSE16}
Zhao, S., Chaturapruek, S., Sabharwal, A., Ermon, S.: Closing the gap between
  short and long xors for model counting. In: Proc. of AAAI (2016)

\end{thebibliography}
\end{document}